\documentclass[nohyperref]{article}

\usepackage{microtype}
\usepackage{graphicx}
\usepackage{subcaption}
\usepackage{booktabs} 
\usepackage{multirow} 
\usepackage{multicol}
\usepackage{hyperref}
\usepackage{enumitem}

\usepackage[accepted]{icml2022}

\usepackage{amsmath}
\usepackage{amssymb}
\usepackage{mathtools}
\usepackage{amsthm}

\usepackage[capitalize,noabbrev]{cleveref}
\usepackage{lipsum}
\theoremstyle{plain}
\newtheorem{theorem}{Theorem}[section]

\theoremstyle{definition}

\theoremstyle{remark}

\usepackage[textsize=tiny]{todonotes}
\usepackage{array}
\usepackage{bm}

 \icmltitlerunning{SafeRL-Kit: Evaluating Efficient Reinforcement Learning Methods for  Safe Autonomous Driving}

\begin{document}

\twocolumn[
\icmltitle{SafeRL-Kit: Evaluating Efficient  Reinforcement Learning Methods\\ for Safe Autonomous Driving}



\icmlsetsymbol{equal}{*}

\begin{icmlauthorlist}
\icmlauthor{Linrui Zhang}{equal,yyy}
\icmlauthor{Qin Zhang}{equal,yyy}
\icmlauthor{Li Shen}{comp}
\icmlauthor{Bo Yuan}{yyy}
\icmlauthor{Xueqian Wang}{yyy}
\end{icmlauthorlist}

\icmlaffiliation{yyy}{Tsinghua Shenzhen International Graduate School, Tsinghua University, Beijing, China.}
\icmlaffiliation{comp}{JD Explore Academy, Beijing, China}

\icmlcorrespondingauthor{Xueqian Wang}{wang.xq@sz.tsinghua.edu.cn}

\icmlkeywords{Machine Learning, ICML}

\vskip 0.3in
]



\printAffiliationsAndNotice{\icmlEqualContribution} 

\begin{abstract}
Safe reinforcement learning (RL) has achieved significant success on risk-sensitive tasks and shown promise in autonomous driving (AD) as well. Considering the distinctiveness of this community, efficient and reproducible baselines are still lacking for safe AD. In this paper, we release SafeRL-Kit to benchmark safe RL methods for AD-oriented tasks. Concretely, SafeRL-Kit contains several latest algorithms specific to zero-constraint-violation tasks, including Safety Layer, Recovery RL, off-policy Lagrangian method, and Feasible Actor-Critic. In addition to existing approaches, we propose a novel first-order method named Exact Penalty Optimization (EPO) and sufficiently demonstrate its capability in safe AD. All algorithms in SafeRL-Kit are implemented (i) under the off-policy setting, which improves sample efficiency and can better leverage past logs; (ii) with a unified learning framework, providing off-the-shelf interfaces for researchers to incorporate their domain-specific knowledge into  fundamental safe RL methods. Conclusively, we conduct a comparative evaluation of the above algorithms in SafeRL-Kit and shed light on their efficacy for safe autonomous driving. The source code is available at \href{ https://github.com/zlr20/saferl_kit}{this https URL}. 
\end{abstract}
\section{Introduction}
\label{sec1}

 \begin{figure*}
      \centering
        \includegraphics[width=0.775\linewidth]{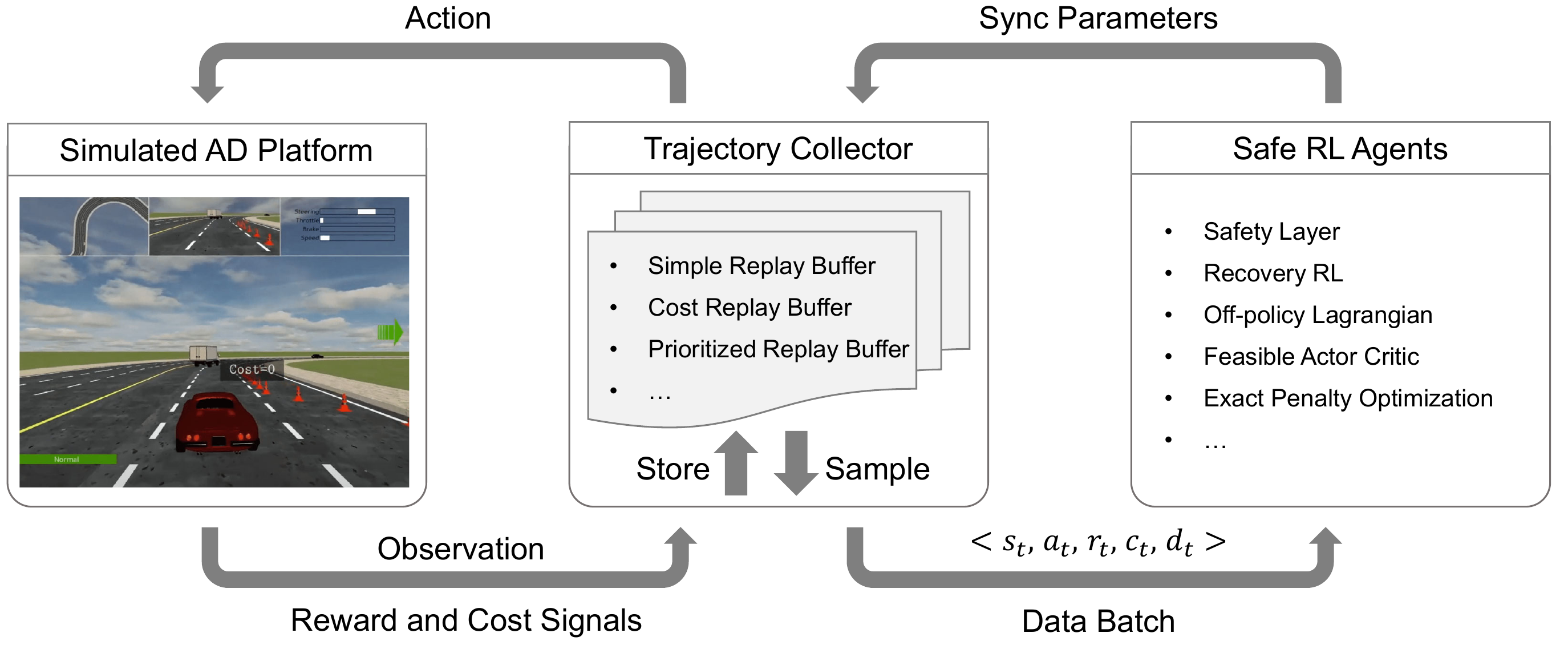}
      \caption{The overall framework of SafeRL-Kit. The trajectory collector interacts with specified AD environments (e.g., MetaDrive~\citep{li2021metadrive}) and stores transitions in the memory. SafeRL-Kit contains several safe RL agents that efficiently learn from past experiences, including Safety Layer, Recovery RL, Off-policy Lagrangian, Feasible Actor Critic, and newly proposed Exact Penalty Optimization.}
      \label{fig:framework}
      \vspace{-0.15cm}
\end{figure*}

Reinforcement Learning (RL) has achieved superhuman performance in many decision-making problems~\citep{mnih2015human,vinyals2019grandmaster}. Typically, the agent learns from trial and error and requires minimal prior knowledge of the environment. Such a paradigm has natural advantages in mastering complex skills for highly nonlinear systems like autonomous vehicles~\citep{kiran2021deep}.

Nevertheless, concerns about the systematic safety limit the widespread use of standard RL in real-world applications~\citep{amodei2016concrete}.  As an alternative, safe RL takes safety requirements as hard constraints and optimizes policies in the feasible domain. In recent years, it has been deemed as a practical solution to resource allocation~\citep{liu2021resource}, robotic locomotion~\citep{yang2022safe}, etc. 

There have also been studies introducing safe RL into autonomous driving (AD)~\citep{isele2018safe,chen2021safe,li2022efficient}. Despite those ongoing efforts, a unified benchmark is of great relevance to facilitate further research on safe AD. We notice some risk-sensitive simulated environments~\citep{li2021metadrive,herman2021learn} have been proposed, but an efficient safe RL toolkit is still absent for this community. Considering the distinctiveness of AD-oriented tasks,  common code-bases~\citep{ray2019benchmarking,yuan2021safe} lack the following pivotal characteristics:
\begin{description}[leftmargin=0em]
    \item [(1) Being safety-critical.] The agent must maintain zero cost-return as much as possible since any inadmissible behavior in autopilot leads to catastrophic failures. Instead, the previous code-base is built for a general-purpose with trajectory-based constraints and non-zero thresholds.
    \item [(2) Being sample-efficient.] Off-policy algorithms can better leverage past logs and human demonstrations, which is crucial for AD. By contrast, the previous code-base requires tens of millions of interactions due to its on-policy algorithms, like CPO and PPO-L~\citep{ray2019benchmarking}.
    \item [(3) Being up-to-date.] There has been a fast-growing body of RL-based safe control. Nevertheless, the previous code-base merely contains elder baselines~\cite{achiam2017constrained,chow2017risk} and lacks the latest advances.
    \item [(4) Being easy-to-use.] Most work on learning-based safe AD tends to incorporate domain-specific knowledge into fundamental safe RL. Thus the toolkit is supposed to provide off-the-shelf interfaces for extended studies. However, the modules of the previous code-base are highly coupled and are implemented with the deprecated TensorFlow version.
\end{description}

To provide a such as toolkit for safe RL algorithms and understand which of them are best suited for AD-oriented tasks, our contributions in this work are summarized as the following three-folds:
\begin{itemize}
    \item We release SafeRL-Kit, which contains the latest advances in safe RL~\cite{dalal2018safe,ha2020learning,thananjeyan2021recovery,ma2021feasible}. All algorithms are implemented efficiently under off-policy settings and with a unified training framework.
    \item We propose a novel first-order method coined Exact Penalty Optimization (EPO) and incorporate it into SafeRL-Kit. EPO utilizes a single penalty factor and a ReLU operator to construct an equivalent unconstrained objective. Empirical results show the simple technique is surprisingly effective and robust for AD-oriented tasks. 
    \item We benchmark SafeRL-Kit in a representative toy environment and a simulated platform with realistic vehicle dynamics. To the best of our knowledge, this paper is the first to provide unified off-policy safe RL baselines and a fair comparison of them specific to AD.
\end{itemize}

\section{Related Work}
\subsection{Safe RL Algorithms}
A number of works tackle RL-based safe control for autonomous agents, and we divide them into three genres. The first type of method, coined as safe policy optimization, incorporates safety constraints into the standard RL objective and yields a constrained sequential optimization problem~\citep{chow2017risk,achiam2017constrained,zhang2020first,ma2021feasible,Zhang2022PenalizedPP}.  The second type of method, coined as safety correction, projects initial unsafe behaviors to the feasible region~\citep{dalal2018safe,zhao2021model}. The third type of method, coined as safety recovery, learns an additional pair of safe actor-critic to take over control when encountering potential risks~\citep{thananjeyan2021recovery,yang2022safe}.

There have also been studies on safe RL specific to AD-oriented tasks. \citet{isele2018safe} utilizes a prediction module to generate masks on dangerous behaviors, which merely works in discrete action spaces. \citet{wen2020safe} extend Constrained Policy Optimization (CPO)~\citep{achiam2017constrained} to AD  and employ synchronized parallel actors to accelerate the convergence speed for on-policy CPO. \citet{chen2021safe} take the ego-camera view as input and train an additional recovery policy via a heuristic objective based on Hamilton-Jacobi reachability. \citet{li2022efficient} propose a human-in-loop approach to learn safe driving efficiently.

\subsection{Safe RL Benchmarks}
For general scenarios, a set of benchmarks are commonly used to evaluate the efficacy of safe RL algorithms.
The classic environments\footnote{https://github.com/SvenGronauer/Bullet-Safety-Gym} include Robot with Limit Speed~\citep{zhang2020first}, Circle and Gather~\citep{achiam2017constrained}, etc. 
Safety-gym\footnote{https://github.com/openai/safety-gym}~\citep{ray2019benchmarking} contains several tasks (goal, button, push) and agents (point, car, doggo) that are representative in robot control problems. Meanwhile, the authors provide popular baselines\footnote{https://github.com/openai/safety-starter-agents}, including CPO and some on-policy Lagrangian methods.
Safe-control-gym\footnote{https://github.com/utiasDSL/safe-control-gym}~\citep{yuan2021safe} bridges the gap between control and RL communities. The authors also developed an open-sourced toolkit supporting both model-based and data-driven control techniques.

For AD-oriented tasks, there have been some existing environments for safe driving. \citet{li2021metadrive} release Metadrive\footnote{https://github.com/metadriverse/metadrive} that benchmarks reinforcement learning algorithms for vehicle autonomy, including safe exploitation and exploration. \citet{herman2021learn} propose Learn-to-Race\footnote{https://github.com/learn-to-race/l2r} that focuses on safe control in high speed. Nevertheless, it still lacks a set of strong baselines specific to the AD community considering its distinctiveness depicted above in Section~\ref{sec1}. To our best knowledge, this paper is the first to provide unified off-policy safe RL baselines and a fair comparison of them for the purpose of autonomous driving.
\section{Preliminaries}
\begin{figure}
    \centering
    \subcaptionbox{Cost Signal = 0\label{fig:env1_safe}}
    {\includegraphics[width=0.475\linewidth,height = 0.36\linewidth]{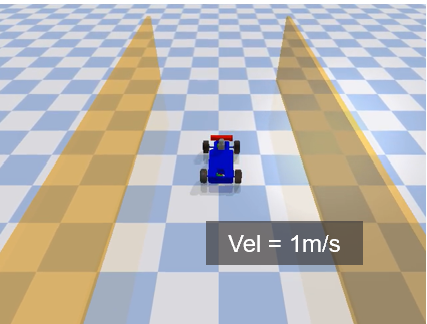}}
    \subcaptionbox{Cost Signal = 1\label{fig:env1_unsafe}}
    {\includegraphics[width=0.475\linewidth,height = 0.36\linewidth]{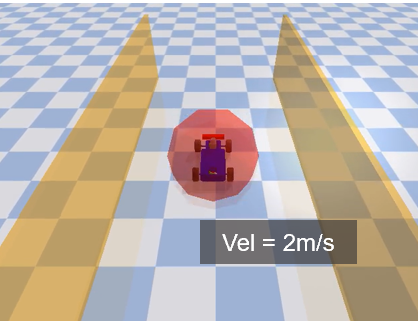}}
    \caption{SpeedLimit Benchmark. The vehicle is rewarded for driving along the avenue, but receives a cost signal if $vel > 1.5 m/s$.}
    \label{fig:env1}
    \vspace{-0.2cm}
\end{figure}

A Markov Decision Process (MDP)~\citep{sutton2018reinforcement} is defined by a tuple $(\mathcal{S},\mathcal{A},\mathcal{P},\mathcal{R},\mu,\gamma)$. ${\mathcal{S}}$ and ${\mathcal{A}}$ denote the state space and the action space respectively. ${\mathcal{P}}: {\mathcal{S}} \times \mathcal{A} \times  \mathcal{S} \mapsto [0,1]$ is the transition  probability  function to describe the dynamics of the system.  $\mathcal{R} : \mathcal{S} \times \mathcal{A} \mapsto \mathbb{R}$ is the reward function. 
$\mu:\mathcal{S} \mapsto [0,1]$ is the initial state distribution. $\gamma$ is the discount factor for future reward.
A stationary policy $\pi : S \mapsto P(A)$ maps the given states to probability distributions over action space.
The goal of standard RL is to find the optimal policy $\pi^*$ that maximizes the expected discounted return
$J_R(\pi) = \mathop{\mathbb{E}}_{\tau\sim \pi}\big [ \sum^\infty_{t=0}\gamma^t R(s_t,a_t)\big ],$ where $\tau=\{(s_t,a_t)\}_{t\ge0}$ is a sample trajectory and $\tau\sim\pi$ accounts for the distribution over trajectories depending on $s_0 \sim \mu, a_t \sim \pi(\cdot | s_t), s_{t+1} \sim P(\cdot | s_t,a_t)$.

A Constrained Markov Decision Process (CMDP)~\citep{altman1999constrained} extends MDP to $(\mathcal{S},\mathcal{A},\mathcal{P},\mathcal{R},\mathcal{C},\mu,\gamma)$.  The cost function $\mathcal{C} : \mathcal{S} \times \mathcal{A} \mapsto [0,+\infty] $ reflects the violation of systematic safety. 
The goal of safe RL is to find \[\pi^* = {\arg\max}_\pi J_R(\pi)\quad \mathrm{s.t.}\ \ \{a_t\}_{t\ge0} \text{ is feasible}.\]

In a CMDP, the cost function is typically constrained in the following two ways.
The first is  \emph{Instantaneous Constrained MDP}. This type of Safe RL formualtion requires the selected actions enforce the constraint at every decision-making step, namely
$C(s_t,a_t)\leq \epsilon$.
The second is \emph{Cumulative Constrained MDP}. This type of Safe RL formualtion requires the discounted sum of cost signals in the whole trajectory within a certain threshold, namely
$
J_C(\pi) = \mathop{\mathbb{E}}_{\tau\sim \pi}\big [ \sum^\infty_{t=0}\gamma^t C(s_t,a_t)\big ] \leq d.
$

\section{Problem Setup}

In this paper, we develop SafeRL-Kit to evaluate efficient RL algorithms for safe autonomous driving on existing benchmarks. We simplify the cost function as the following risk-indicator:
\begin{equation}
    C(s,a) = \begin{cases}
            1,& \text{if the transition is unsafe}\\
            0,  & \text{otherwise}
        \end{cases}.
\end{equation}
This formulation is generalizable to different AD-oriented tasks without cumbersome  reward and cost shaping. The goal of the autonomous vehicle is to reach the destination as fast as possible while adhering to zero cost signals at every time steps. Specifically, we conduct comparative evaluations on a representative toy environment and a simulated platform with realistic vehicle dynamics respectively.

\begin{figure}
    \centering
    \subcaptionbox{Cost Signal = 0\label{fig:env2_safe}}
    {\includegraphics[width=0.475\linewidth,height = 0.36\linewidth]{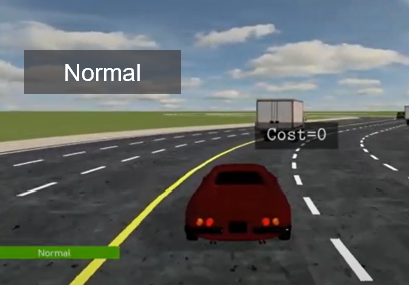}}
    \subcaptionbox{Cost Signal = 1\label{fig:env2_unsafe}}
    {\includegraphics[width=0.475\linewidth,height = 0.36\linewidth]{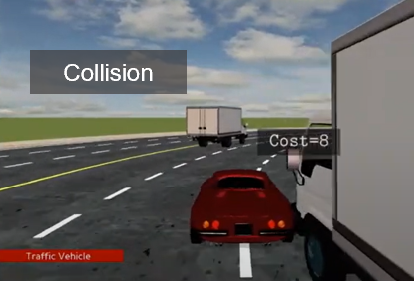}}
    \caption{MetaDrive Benchmark. The vehicle aims to reach virtual markers, but receives a cost signal if it collides with obstacles and other vehicles or it is out of the road.}
    \label{fig:env2}
    \vspace{-0.1cm}
\end{figure}

\begin{table*}[t]
\caption{Comparison of different safe reinforcement learning algorithms for AD-oriented tasks.}
\label{tab:comparision}

\begin{center}
\begin{small}
\begin{sc}
\resizebox{\textwidth}{21.5mm}{
\begin{tabular}{|c|cc|cc|}
\toprule
\specialrule{0em}{1pt}{1pt}
\multirow{2}*{Algorithm} & \multicolumn{2}{c|}{Constraint Type} &  \multicolumn{2}{c|}{Policy Type} \\
\specialrule{0em}{1pt}{1pt}
\cline{2-5}
\specialrule{0em}{1pt}{1pt}
~ & Cumulative/Instantaneous & State-wise/Trajectory-wise & Deterministic & Stochastic \\
\specialrule{0em}{1pt}{1pt}
\midrule
\specialrule{0em}{1pt}{1pt}
CPO~\citep{ray2019benchmarking} & Cumulative & Trajectory-wise  & $\times$ & $\surd$\\
\specialrule{0em}{1pt}{1pt}
PPO-L~\citep{ray2019benchmarking} & Cumulative & Trajectory-wise & $\times$ & $\surd$\\
\specialrule{0em}{1pt}{1pt}
TRPO-L~\citep{ray2019benchmarking} & Cumulative & Trajectory-wise & $\times$ & $\surd$ \\
\specialrule{0em}{1pt}{1pt}
Safety Layer & Instantaneous & State-wise & $\surd$ & $\times$\\
\specialrule{0em}{1pt}{1pt}
Recovery RL & Cumulative & State-wise & $\surd$ & $\surd$\\
\specialrule{0em}{1pt}{1pt}
Off-policy Lagrangian & Cumulative & Trajectory-wise & $\surd$ & $\surd$\\
\specialrule{0em}{1pt}{1pt}
Feasible Actor-Critic & Cumulative & State-wise & $\surd$ & $\surd$\\
\specialrule{0em}{1pt}{1pt}
Exact Penalty Optimization & Cumulative & Both & $\surd$ & $\surd$\\
\specialrule{0em}{1pt}{1pt}
\bottomrule
\end{tabular}
}
\end{sc}
\end{small}
\end{center}
\vskip -0.1in
\end{table*}

\subsection{SpeedLimit Benchmark}

The task is inspired by~\citet{zhang2020first}, as illustrated in Figure~\ref{fig:env1}. In SpeedLimit task, the agent is a four-wheeled race-car whose observation is ego position, velocity and yaw. The selected action controls the Revolution Per Minute (RPM) and steering of wheels. The agent is rewarded for approaching  $x_{dest}=+\infty$ and the cost function is
\begin{equation}
    C(s,a) = \begin{cases}
            1,& \text{if vehicle's velocity} > 1.5m/s\\
            0,  & \text{otherwise}
        \end{cases}.
\end{equation}

The toy environment is simple yet representative since speed control is a classic problem in vehicle autonomy. Besides, the speed limit is easy to reach and thus undesirable algorithms may violate the safety constraint at almost every time step. That is, the toy environment enables us to see which algorithms can effectively degrade the dense cost return and  are best suited for safe AD tasks.

\subsection{MetaDrive Benchmark}

This task is inspired by~\citet{li2021metadrive}, as illustrated in Figure~\ref{fig:env2}. Metadrive is a compositional, lightweight and  realistic platform for vehicle autonomy. Most importantly, it provides pre-defined environments for safe policy learning in autopilots. Concretely, the observation is encoded by a vector containing ego-state, navigation information and surrounding information detected by the Lidar. We control the speed and steering of the car to hit virtual land markers for rewards, and the cost function is defined as
\begin{equation}
    C(s,a) = \begin{cases}
            1,& \text{if collides or out of the road}\\
            0,  & \text{otherwise}
        \end{cases}
\end{equation}

It worth mentioning that we set the traffic density twice than the original paper to construct a more challenging scenario.

\section{Efficient Safe RL Algorithms}
\label{sec5}
\subsection{Overall Implementation}
The current version of SafeRL-Kit contains some latest RL-based methods, including \emph{Safety Layer}~\citep{dalal2018safe}, \emph{Recovery RL}~\citep{thananjeyan2021recovery}, \emph{Off-policy Lagrangian}~\citep{ha2020learning}, \emph{Feasible Actor-Critic}~\citep{ma2021feasible} and newly proposed \emph{Exact Penalty Optimization}. We compare above methods along with some existing on-policy baselines~\citep{ray2019benchmarking} in Table~\ref{tab:comparision}.

Before diving into algorithmic details, we first explain the overall implementation of SafeRL-Kit and its benefits:

\begin{description}[leftmargin=0em]
\item[(1) ] The adopted algorithms address safe policy learning from different perspectives (Safety Layer for safety correction; Recovery RL for safety recovery; Lagrangian, FAC, and EPO for constrained optimization). Thus, users can combine AD-specific knowledge with the proper type of safe RL baselines in their studies.
\item[(2) ] All the algorithms are implemented under the off-policy Actor-Critic architecture. Thus, they enjoy better sample efficiency and can leverage human demonstration if needed.
\item[(3) ] All the algorithms are implemented with a unified training framework. By default, all networks are MLPs with (256,256) hidden layers activated via the ReLU function. The essential updates of backbone networks follow TD3~\citep{fujimoto2018addressing} without pre-training processes. Thus, we can conduct a fair comparison to see which of them are best suited for AD-oriented tasks.
\end{description}

\subsection{Safety Layer}


Safety Layer, added on top of the original policy network, conducts a quadratic-programming-based constrained optimization to find the "nearest" action to the feasible region. 

Specifically, Safety Layer utilizes a parametric linear model
\begin{equation}
    C(s_t,a_t) \approx g(s_t;\omega)^\top a_t + c_{t-1}
\end{equation}
to approximate the single-step cost function with supervised training and yields the following QP problem
\begin{equation}
\begin{aligned} \label{SafetyLayer}
a_t^* = &\ \  {\arg\min}_{a}\ \frac{1}{2}|| a - \mu_\theta(s)||^2\\
& \mathrm{s.t.} \quad g(s_t;\omega)^\top a_t + c_{t-1} \le \epsilon, 
\end{aligned}
\end{equation}
which projects the unsafe action back to the feasible region.

Since there is only one compositional cost signal in our problem, the closed-form solution of problem~\eqref{SafetyLayer} is
\begin{equation}
        a_t^* = \mu_\theta(s_t) - \bigg[\frac{g(s_t;\omega)^\top\mu_\theta(s) + c_{t-1} - \epsilon}{g(s_t;\omega)^\top g(s_t;\omega)} \bigg]^+ g(s_t;\omega)
\end{equation}

Thus, Safety Layer is the type of method that addresses state-wise, instantaneous constraints.

By the way, the $g_\omega$ is trained from offline data in~\citet{dalal2018safe}. SafeRL-Kit instead learns the linear model with the policy network synchronously, considering the side-effect of distribution shift. We employ a warm-up in the training process to avoid meaningless, inaccurate corrections.

\begin{table*}[t]
\caption{Hyper-parameters of different safety-aware algorithms in SafeRL-Kit.}
\label{tab:hyper}
\begin{center}
\begin{small}
\begin{sc}
\resizebox{\textwidth}{33mm}{
\begin{tabular}{|l|p{2.25cm}<{\centering}p{2.25cm}<{\centering}p{2.25cm}<{\centering}p{2.25cm}<{\centering}p{2.25cm}<{\centering}|}
\toprule
 Hyper-parameter& Safety Layer &  Recovery RL & Lagrangian & FAC & EPO \\
\midrule
\specialrule{0em}{1pt}{1pt}
Cost Limit & 0.02 & 0.1  & 0.1 & 0.1 & 0.1\\
\specialrule{0em}{1pt}{1pt}
Reward Discount & 0.99 & 0.99  & 0.99 & 0.99 & 0.99\\
\specialrule{0em}{1pt}{1pt}
Cost Discount & 0.99 & 0.99  & 0.99 & 0.99 & 0.99\\
\specialrule{0em}{1pt}{1pt}
Warm-up Ratio & 0.2 & 0.2  & N/A & N/A & N/A\\
\specialrule{0em}{1pt}{1pt}
Batch Size & 256 &  256   &  256  &  256  & 256 \\
\specialrule{0em}{1pt}{1pt}
Critic LR & 3E-4& 3E-4& 3E-4& 3E-4& 3E-4\\
\specialrule{0em}{1pt}{1pt}
Actor LR & 3E-4& 3E-4& 3E-4& 3E-4& 3E-4\\
\specialrule{0em}{1pt}{1pt}
Safe Critic LR & 3E-4& 3E-4& 3E-4& 3E-4& 3E-4\\
\specialrule{0em}{1pt}{1pt}
Safe Actor LR &  N/A& 3E-4&  N/A&  N/A&  N/A\\
\specialrule{0em}{1pt}{1pt}
Multiplier LR &  N/A & N/A & 1E-5& 1E-5&  N/A \\
\specialrule{0em}{1pt}{1pt}
Multiplier Init &  N/A & N/A & 0.0 & N/A &  N/A \\
\specialrule{0em}{1pt}{1pt}
Policy Delay & 2 &  2   &  2  &  2  & 2 \\
\specialrule{0em}{1pt}{1pt}
Multiplier Delay &  N/A &   N/A   &   N/A  &  12  &  N/A\\
\specialrule{0em}{1pt}{1pt}
Penalty Factor &  N/A &   N/A   &   N/A    &  N/A & 5\\
\bottomrule
\end{tabular}
}
\end{sc}
\end{small}
\end{center}
\vskip -0.1in
\end{table*}

\subsection{Recovery RL}

The critical insight behind Recovery RL is to introduce an additional policy that recovers potential unsafe states. Consequently, it trains two independent RL agents instead of solving a cumbersome constrained optimization problem. 

Specifically, Recovery RL learns a safe critic to estimate the future probability of constraint violation as
\begin{equation}\label{Q_risk}
    Q^\pi_\text{risk}(s_t,a_t) = c_t + (1-c_t) \gamma \mathbb{E}_{\pi} Q^\pi_\text{risk}(s_{t+1},a_{t+1}).
\end{equation}
This formulation is slightly different from the standard Bellman equation since it assumes the episode terminates when the agent receives a cost signal. We found in experiments that such an early stopping makes it intractable for agents to master desirable skills in AD. Thus, we remove the early-stopping condition but still preserve the original formulation of $Q^\pi_\text{risk}$ in~\eqref{Q_risk} since it limits the upper bound of the safe critic and eliminates the over-estimation in Q-learning. 

In the phase of policy execution, the recovery policy takes over the control when the predicted value of the safe critic exceeds the given threshold, as
\begin{equation}
a_t = 
        \begin{cases}
            \pi_\text{task}(s_t),& \text{if } Q^\pi_\text{risk}\big(s_t,\pi_\text{task}(s_t)\big)\leq \epsilon\\
            \pi_\text{risk}(s_t),  & \text{otherwise}
        \end{cases}
\end{equation}

The optimization objective of $\pi_\text{task}$ is to maximize the cumulative rewards, whereas the goal of $\pi_\text{risk}$ is to minimize $Q^\pi_\text{risk}$, namely to degrade the potential risk of the agent.

It is important to store $a_\text{task}$ and $a_\text{risk}$ simultaneously in the replay buffer, and utilize them to train $\pi_\text{task}$ and $\pi_\text{risk}$ respectively in Recovery RL. This technique ensures that $\pi_\text{task}$ can learn from the new MDP, instead of proposing same unsafe actions continuously.

Similar to Safety Layer, Recovery RL in SafeRL-Kit also has a warm-up stage where $Q^\pi_\text{risk}$ is trained but is not utilized.

\subsection{Off-policy Lagrangian}
\label{sec5.3}
Lagrangian Relaxation is commonly-used to address constrained optimization problem. Safe RL as well can be formulated as a constrained sequential optimization problem
\begin{equation}
\begin{aligned} \label{LAG}
    \mathop{\max}_{\pi} &\mathop{\mathbb{E}}_{s\sim\mu} V_0^\pi(s)\\
    \mathrm{s.t.} \ \  &\mathop{\mathbb{E}}_{s\sim\mu} U^\pi_{0}(s) \leq \epsilon,
\end{aligned}
\end{equation}
where $V^\pi_0(s) = \mathop{\mathbb{E}}_{\tau\sim \pi}\big [ \sum^\infty_{t=0}\gamma^t r_t\big |s_0 = s ] $ and $U^\pi_0(s) = \mathop{\mathbb{E}}_{\tau\sim \pi}\big [ \sum^\infty_{t=0}\gamma^t c_t\big |s_0 = s ] $.

Strong duality holds for primal problem~\eqref{LAG}~\citep{paternain2022safe}, thus it can be tackled via the dual problem
\begin{equation}\label{LAGDual}
    \mathop{\max}_{\lambda \geq 0} \mathop{\min}_{\pi} \mathop{\mathbb{E}}_{s\sim\mu} -V_0^\pi(s) + \lambda \big(U^\pi_{0}(s) - \epsilon \big).
\end{equation}

The off-policy objective of problem~\eqref{LAGDual} in the  parametric space~\citep{ha2020learning} can be formulated as
\begin{equation}
    \mathop{\max}_{\lambda \geq 0} \mathop{\min}_{\theta} \mathbb{E}_{\mathcal{D}} -Q^\pi(s,\pi_\theta(s)) + \lambda \big(Q^\pi_{c}(s,\pi_\theta(s)) - \epsilon \big).
\end{equation}

Stochastic primal-dual optimization~\cite{luenberger1984linear} is applied here to update primal and dual variables alternatively, which follows as
\begin{equation}\label{pd}
        \begin{cases}
            \theta \leftarrow \theta - \eta_\theta \nabla_\theta \mathbb{E}_{\mathcal{D}}\big( -Q^\pi(s,\pi_\theta(s)) + \lambda Q^\pi_{c}(s,\pi_\theta(s))  \big)\\
            \lambda \leftarrow \big[ \lambda + \eta_\lambda \mathbb{E}_{\mathcal{D}}\big(Q^\pi_{c}(s,\pi_\theta(s)) - \epsilon \big) \big]^+
        \end{cases}
\end{equation}

Notably, the timescale of primal variable updates is required to be faster than the timescale of Lagrange multipliers. Thus, we set $\eta_\theta \gg \eta_\lambda$ in SafeRL-Kit.

\subsection{Feasible Actor-Critic}
\label{sec5.4}
The constraint of Off-policy Lagrangian in Section~\ref{sec5.3} is based on the expectation of whole trajectories, which inevitably allows some unsafe roll-outs. \citet{ma2021feasible} introduce a new concept, namely state-wise constraints for cumulative cost-return which follows as
\begin{equation}
\begin{aligned} \label{FAC}
    \mathop{\max}_{\pi} &\mathop{\mathbb{E}}_{s\sim\mu} V_0^\pi(s)\\
    \mathrm{s.t.} \ \  & U^\pi_{0}(s) \leq \epsilon, \forall s \in \mathcal{I_F}.
\end{aligned}
\end{equation}
Here $s \in \mathcal{I_F}$ stands for all "feasible" initial states.  Also, their theoretical results show that problem~\eqref{FAC} is a stricter version than problem~\eqref{LAG}, which provides strong safety guarantees for state-wise safe control. 

The dual problem of~\eqref{FAC} is derived by rescaling the state-wise constraints and follows as
\begin{equation}\label{FACDual}
    \mathop{\max}_{\lambda \geq 0} \mathop{\min}_{\pi} \mathop{\mathbb{E}}_{s\sim\mu} -V_0^\pi(s) + \lambda(s) \big(U^\pi_{0}(s) - \epsilon \big).
\end{equation}

The distinctiveness of problem~\eqref{FACDual} is there are infinitely many Lagrangian multipliers that are state-dependent. In SafeRL-Kit, we employ an neural network $\lambda(s;\xi)$ activated by \emph{Softplus} function to map the given state $s$ to its corresponding Lagrangian multiplier $\lambda(s)$.

The primal-dual ascents of policy network is similar to~\eqref{pd}; the updates of multiplier network is given by
\begin{equation}
    \xi \leftarrow \xi + \eta_\xi \nabla_\xi  \mathbb{E}_{\mathcal{D}} \lambda(s;\xi) \big(Q^\pi_{c}(s,\pi_\theta(s)) - \epsilon \big).
\end{equation}

Besides, SafeRL-Kit also sets a different interval schedule $m_\pi$ (for $\pi_\theta$
delay steps) and $m_\lambda$  (for $\lambda_\xi$ delay steps) to stabilize the training process~\citep{ma2021feasible}.

\subsection{Exact Penalty Optimization}

\begin{algorithm}[tb]
\caption{State-wise Exact Penalty Optimization}
\begin{algorithmic}[1]
\label{algo1}
\REQUIRE deterministic policy network $\pi(s;\theta)$; critic networks $\hat{Q}(s,a;\phi)$ and $\hat{Q}_c(s,a;\varphi)$
\FOR{t \textbf{in} $1,2,...$}
\STATE $a_t = \pi(s_t;\theta) + \epsilon,\ \ \epsilon\sim\mathcal{N}(0,\sigma)$.
\STATE Apply $a_t$ to the environment.
\STATE Store the transition $(s_t,a_t,s_{t+1},r_t,c_t,d_t)$ in $\mathcal{B}$.
\STATE Sample a mini-batch of $N$ transitions from $\mathcal{B}$.
\STATE $\varphi \leftarrow {\arg\min}_\varphi \mathop{\mathbb{E}}_{\mathcal{B}} \big[\hat Q_c(s,a;\varphi)-\big(c+\gamma_c(1-d) \hat Q_C(s',\pi(s';\theta);\varphi) \big)\big]^2$.
\STATE $\phi \leftarrow {\arg\min}_\phi \mathop{\mathbb{E}}_{\mathcal{B}} \big[\hat Q(s,a;\phi)-\big(r+\gamma(1-d) \hat Q(s',\pi(s';\theta);\phi)\big)\big]^2$.
\STATE $\theta \leftarrow {\arg\min}_\theta \mathop{\mathbb{E}}_{\mathcal{B}}\big[ -\hat Q (s,\pi(s;\theta);\phi) + \kappa\cdot\max\{0, \hat Q_c(s,\pi(s;\theta);\varphi) - \delta\} \big]$.
\ENDFOR
\end{algorithmic}
\end{algorithm}

In this paper, we propose a simple-yet-effective approach motivated by the exact penalty method.

\begin{theorem}\label{th:exact}
Considering the following two problems
    \begin{align}
    &\min f(x)\ \ \mathrm{s.t.} \ g_i(x)\leq0, i=1,2,...\tag{P} \label{p}\\
    & \min f(x) + \kappa \cdot \sum_i \max \{0,g_i(x)\}\tag{Q} \label{q}
    \end{align}
Suppose $\lambda^*$ is the optimal Lagrange multiplier vector of problem (\ref{p}). If the penalty factor $\kappa \geq ||\lambda^*||_\infty$, problem (\ref{p}) and problem (\ref{q}) share the same optimal solution set.
\end{theorem}
\begin{proof}
See our recent work~\citep{Zhang2022PenalizedPP}.
\end{proof}

The above theorem enables us to construct an equivalent function whose unconstrained minimizing points also solve the previous constrained problem. Meanwhile, the unconstrained problem can tackle multiple constraints with exactly one consistent penalty factor.

Thus, we simplify Lagrangian-based methods (i.e., Off-policy Lagrangian and FAC) with this technique, considering that the single-constrained optimization problem~\eqref{LAG} and the multi-constrained optimization problem~\eqref{FAC} are suited for exact penalty method in Theorem~\ref{th:exact}. In this way, we can employ a single minimization on primal variables with fixed penalty terms instead of cumbersome min-max optimization over both primal and dual variables. 

Below we merely summarize the state-wise Exact Penalty Optimization (EPO) in Algorithm~\ref{algo1} as an alternative to FAC, since FAC provides stricter safety guarantees but suffers from the oscillation and instability of the multiplier network. The off-policy surrogate objective of state-wise EPO follows as
\begin{equation}
    \ell(\theta) = \mathbb{E}_{\mathcal{D}} -Q^\pi(s,\pi_\theta(s)) + \kappa \big[Q^\pi_{c}(s,\pi_\theta(s)) - \epsilon \big]^+,
\end{equation}
where $\kappa$ is a fixed, sufficiently large hyper-parameter.

\section{Empirical Analysis}
We benchmark RL-based algorithms on SpeedLimit task~\citep{zhang2020first} and MetaDrive platform~\citep{li2021metadrive}. Below, we give a comparative evaluation according to the empirical results.

\begin{table*}
\centering
\caption{Mean performance with normal 95\% confidence for safety-aware algorithms on benchmarks.}
\vskip 0.1in
\begin{small}
\begin{sc}
\resizebox{\textwidth}{16mm}{
\begin{tabular}{llccccc}
\hline
\specialrule{0em}{1pt}{1pt}
\multicolumn{2}{c}{Environments} & Safety Layer  & Recovery RL &  Lagrangian & FAC & EPO \\
\specialrule{0em}{1pt}{1pt}
\hline
\specialrule{0em}{1pt}{1pt}
\multirow{3}*{SpeedLimit} &  Ep-Reward & $651.59\pm 10.70$ & $623.67\pm 99.58$& $565.50\pm 69.29$ & $631.55\pm 34.92$& $\bm{684.86\pm 3.19}$\\
~ & Ep-Cost & $76.30\pm9.07$ & $187.14\pm96.50$ &$7.28\pm3.11$ & $7.83\pm5.23$& \bm{$5.44\pm0.53$}\\
~ &  CostRate & $0.33\pm0.01$ & $0.43\pm0.06$ &$0.06\pm0.01$ & $0.07\pm0.01$& \bm{$0.02\pm0.01$}\\
\specialrule{0em}{1pt}{1pt}
\hline
\specialrule{0em}{1pt}{1pt}
\multirow{3}*{MetaDrive} &  SuccessRate & $0.73\pm 0.05$&  \bm{$0.78\pm 0.06$}& $0.74\pm0.05$ & $0.68\pm 0.04$ & $0.73\pm 0.05$\\
~ &  Ep-Cost & $12.91\pm1.10$ & $14.18\pm1.92$ &$9.23\pm4.88$ & \bm{$3.29\pm0.50$}& $4.29\pm0.71$\\
~ &  CostRate & $0.04\pm0.001$ & $0.05\pm0.001$ &$0.02\pm0.01$ & \bm{$0.01\pm0.01$}& $0.01\pm0.01$\\
\specialrule{0em}{1pt}{1pt}
\hline
\end{tabular}
}
\end{sc}
\end{small}
\label{tab:perf}
\end{table*}

\paragraph{Unconstrained Reference.}
We utilize TD3~\citep{fujimoto2018addressing} as the unconstrained reference for upper bounds of reward performance and constraint violations. For the SpeedLimit task (500 max\_episode\_horizon), TD3 exceeds the velocity threshold at almost every step with a near 100\% cost rate. For the MetaDrive environment (1000 max\_episode\_horizon), the agent receives sparse cost signals when it collides with obstacles or is out of the road. Besides, the cost signals are encoded into the reward function; otherwise, it would be too hard to learn desirable behaviors~\cite{li2021metadrive}. Consequently, TD3 with reward-shaping (TD3-RS) would not have that high cumulative costs as it does in the toy environment.

\paragraph{Overall Performance.}
The mean performances are summarized in Table~\ref{tab:hyper} and the learning curves over five seeds are shown in Figure~\ref{fig:learing_curves_toy} and \ref{fig:learing_curves}. We conclude that Safety Layer and Recovery RL are less effective in degrading cost return. They still have around 10\% safety violations in SpeedLimit, and the safety improvement in MetaDrive is also limited. As for Safety Layer, the main reasons are that the linear approximation to the cost function brings about non-negligible errors, and the single-step correction is myopic for future risks. As for Recovery RL, the estimation error of $Q_\text{risk}$ is probably the biggest problem affecting the recovery effects. By contrast, Off-policy Lagrangian and FAC have significantly lower cumulative costs. However, the Lagrangian-based methods have the inherent problems from primal-dual ascents. For one thing, the Lagrangian multiplier tuning causes oscillations of learning curves. For another thing, those algorithms are susceptible to Lagrangian multipliers' initialization and learning rate. We conclude that constrained optimization still outperforms safety correction and recovery if the hyper-parameters are appropriately settled. At last, we find that the newly proposed EPO is surprisingly effective for learning safe AD. In SpeedLimit, it converges to a high plateau quickly while adhering to an almost zero cost return. In MetaDrive, it is still competitive with SOTA baselines. We regard the underlying reason as that EPO is an equivalent form to FAC but reduces state-dependent Lagrangian multipliers to one fixed hyper-parameter. The consistent loss function stabilizes the training process compared with primal-dual optimization.

\paragraph{Sensitivity Analysis.} In this paper, we study the sensitivity to hyper-parameters of Lagrangian-based methods and EPO in Figure~\ref{fig:learing_curves_sa1} and Figure~\ref{fig:learing_curves_sa2} respectively. We found that Lagrangian-based methods are susceptible to the learning rate of the Lagrangian multiplier(s) in stochastic primal-dual optimization. First, the oscillating $\lambda$ causes non-negligible deviations in the learning curves. Besides, the increasing $\eta_\lambda$ may degrade the performance dramatically. The phenomenon is especially pronounced in FAC, which has a multiplier network to predict the state-dependent $\lambda(s;\xi)$. Thus, we suggest $\eta_\lambda \ll \eta_\theta$ in practice. As for EPO, we find if the penalty factor $\kappa$ is too small, the cost return may fail to converge. Nevertheless, if $\kappa$ is sufficiently large, the learning curves are robust and almost identical. Thus, we suggest $\kappa > 5$ in experiments and a grid search for better performance.

\paragraph{Sample Complexity.}
Considering the difficulty of the above two tasks, we run $5\times 10^5$ and $1 \times 10^6$ interactive steps respectively to obtain admissible results. Notably, previous on-policy codebases require significantly more samples for convergence; for example, \citet{ray2019benchmarking} run $1\times 10^7$ interactive steps even for toy environments. Thus, SafeRL-Kit with off-policy implementations is 
much more sample-efficient compared to theirs, emphasizing the applicability of SafeRL-Kit to data-expensive AD-oriented tasks. 
\section{Further Discussion}
The released SafeRL-kit contains several SOTA off-policy safe RL methods that are suited for safety-critical autonomous driving. We conduct the comparative evaluation of those baselines over one representative toy environment and one simulated AD platform, respectively. The proposed Exact Penalty Optimization in this paper is easy-to-implement and surprisingly effective on AD-oriented tasks. We think future work on SafeRL-kit from two aspects:
\begin{itemize}
    \item The off-policy implementation of SafeRL-Kit can naturally leverage offline data, including past logs and human demonstrations, which are commonly used and highly effective for AD-oriented tasks.
    \item We only benchmark safe RL methods with vector input (ego-state, navigation information, Lidar signals, etc.) in this paper. Nevertheless, vision-based AD is still less studied in the current version of SafeRL-Kit.
\end{itemize}

\begin{figure*}
      \centering
    \hspace{0.5cm}\includegraphics[width=0.85\linewidth,trim=0 0 0 0,clip]{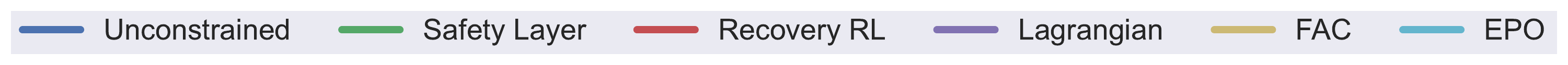}\vspace{0.1cm}\\
    \subcaptionbox{Eval Episode Reward}
        {\includegraphics[width=0.3\linewidth,trim=20 20 0 0,clip]{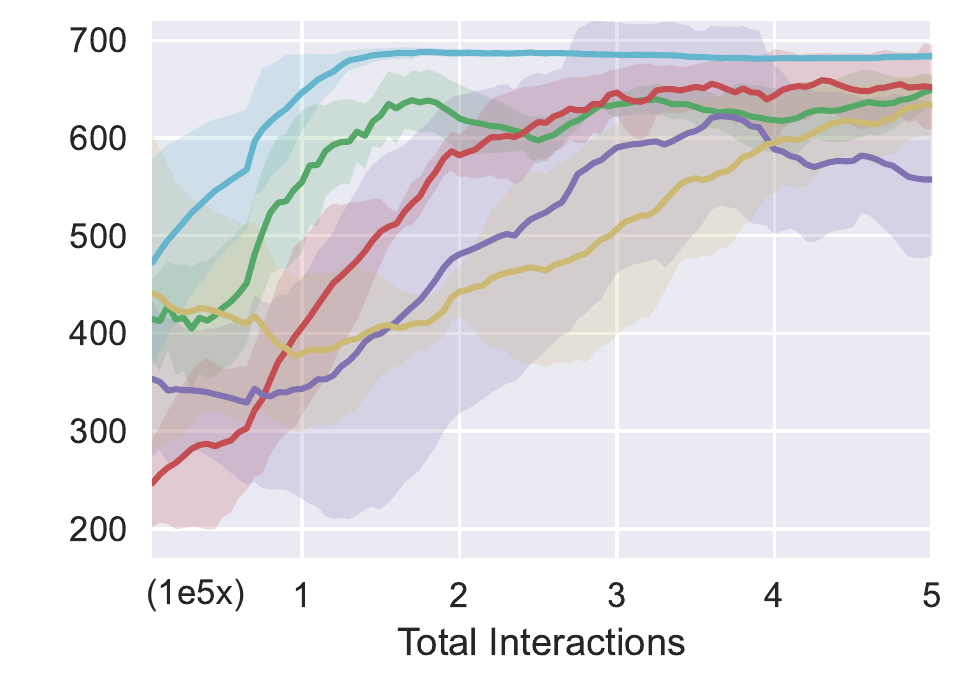}}
       \subcaptionbox{Eval Episode Cost}
        {\includegraphics[width=0.3\linewidth,trim=20 20 0 0,clip]{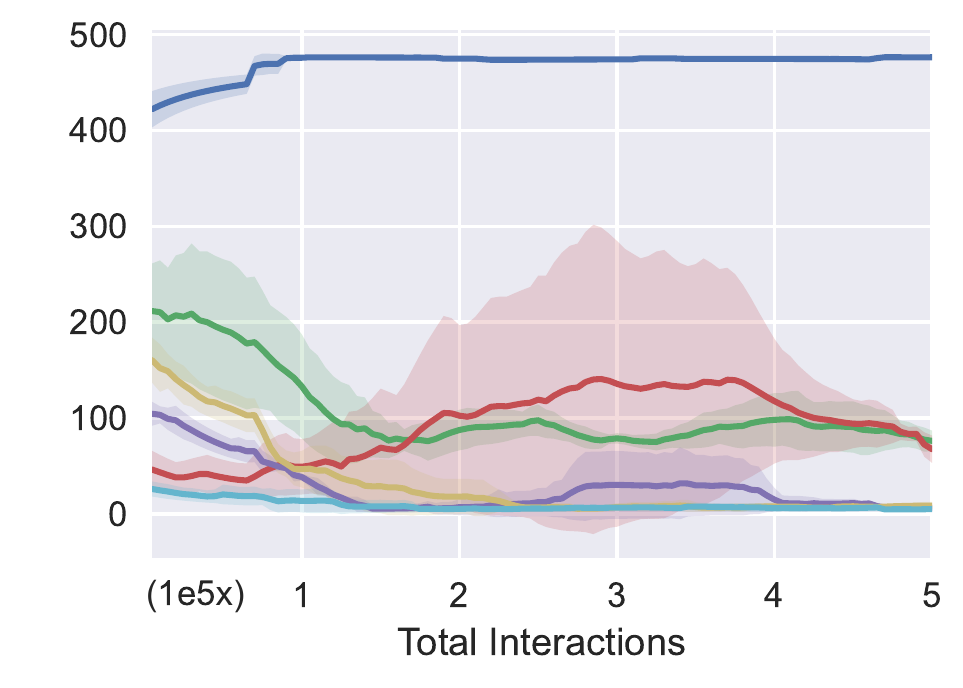}}
      \subcaptionbox{Training Cost rate}
        {\includegraphics[width=0.3\linewidth,trim=20 20 0 0,clip]{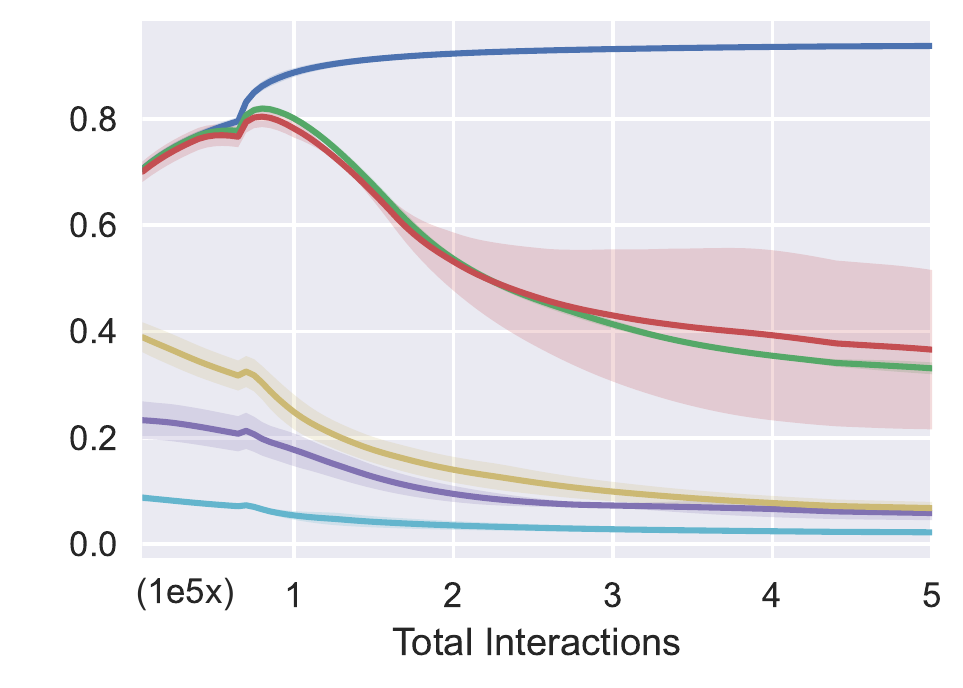}}
        \vspace{-0.15cm}
      \caption{Learning curves on the SpeedLimit benchmark. The x-axis is the number of interactions with the simulator (500,000 total).}
      \label{fig:learing_curves_toy}
      
\end{figure*}

\begin{figure*}
      \centering
    \hspace{0.5cm}\includegraphics[width=0.85\linewidth,trim=0 0 0 0,clip]{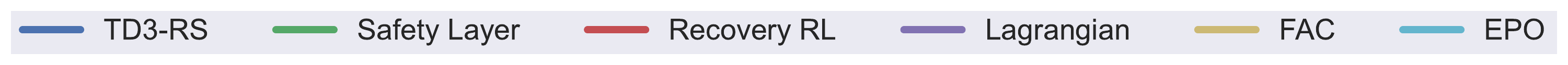}\vspace{0.1cm}\\
    \subcaptionbox{Eval Success Rate}
        {\includegraphics[width=0.3\linewidth,trim=20 20 0 0,clip]{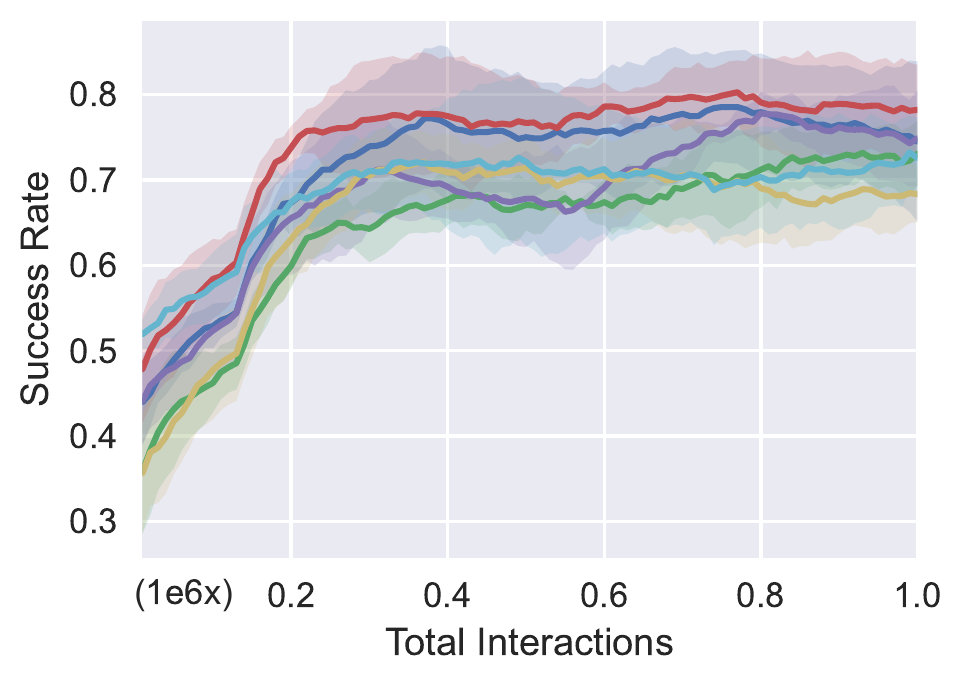}}
       \subcaptionbox{Eval Episode Cost}
        {\includegraphics[width=0.3\linewidth,trim=20 20 0 0,clip]{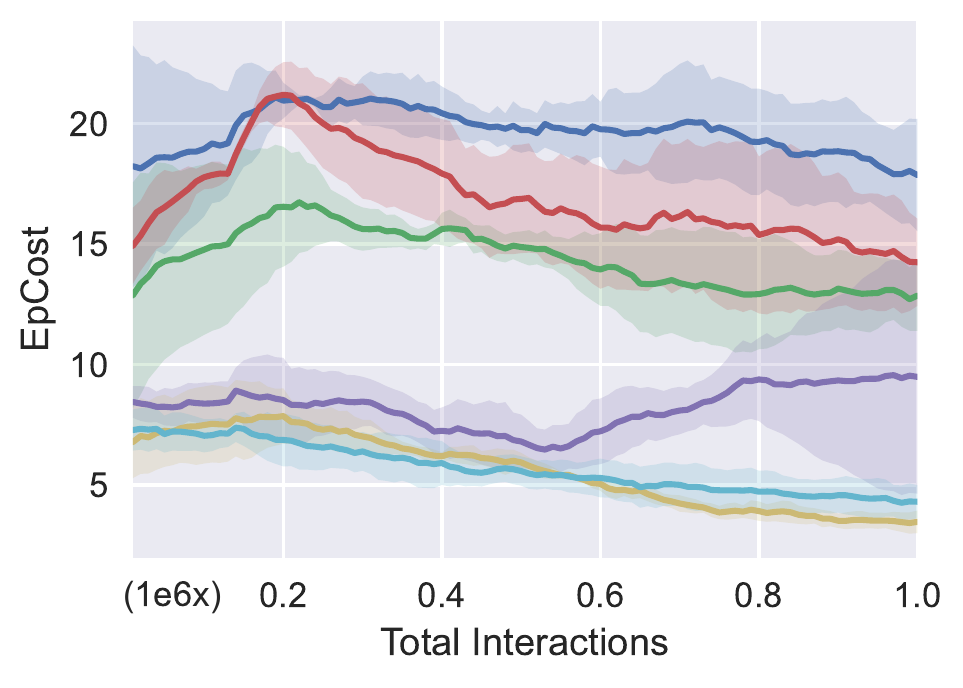}}
      \subcaptionbox{Training Cost rate}
        {\includegraphics[width=0.3\linewidth,trim=20 20 0 0,clip]{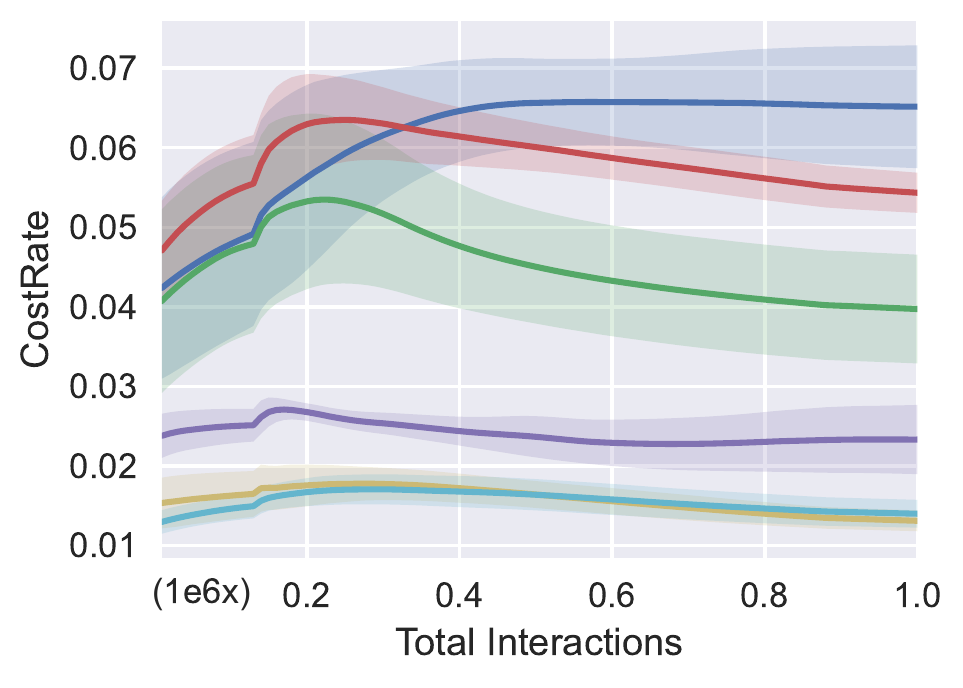}}
      \vspace{-0.15cm}
      \caption{Learning curves on the MetaDrive Benchmark. The x-axis is the number of interactions with the simulator (1,000,000 total).}
      \label{fig:learing_curves}
      \vspace{-0.25cm}
\end{figure*}

\begin{figure*}
      \centering
    \hspace{0.5cm}\includegraphics[width=0.85\linewidth,trim=0 0 0 0,clip]{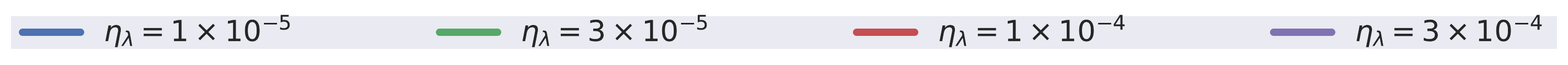}\vspace{0.1cm}\\
    \subcaptionbox{Reward-Lag-SpeedLimit}
        {\includegraphics[width=0.23\linewidth,trim=20 20 0 0,clip]{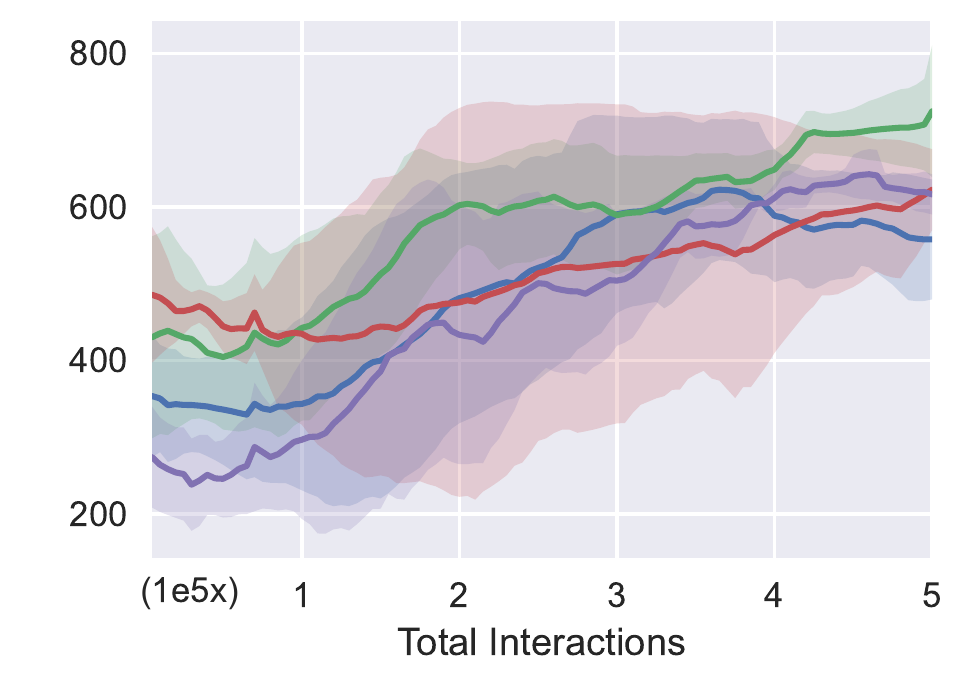}}
       \subcaptionbox{Cost-Lag-SpeedLimit}
        {\includegraphics[width=0.23\linewidth,trim=20 20 0 0,clip]{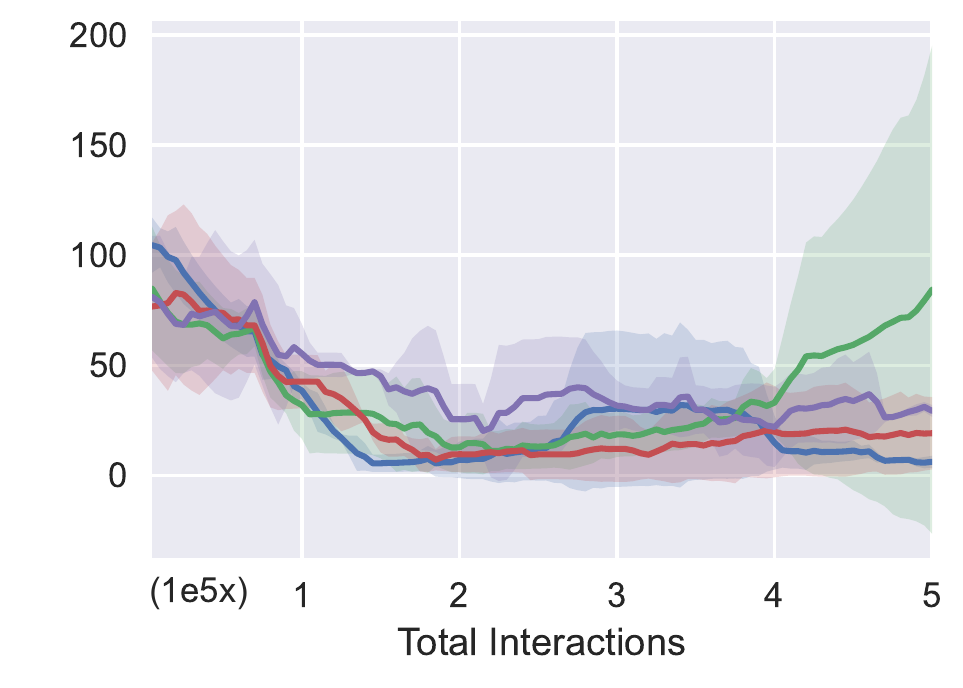}}
      \subcaptionbox{Reward-FAC-MetaDrive}
        {\includegraphics[width=0.23\linewidth,trim=20 20 0 0,clip]{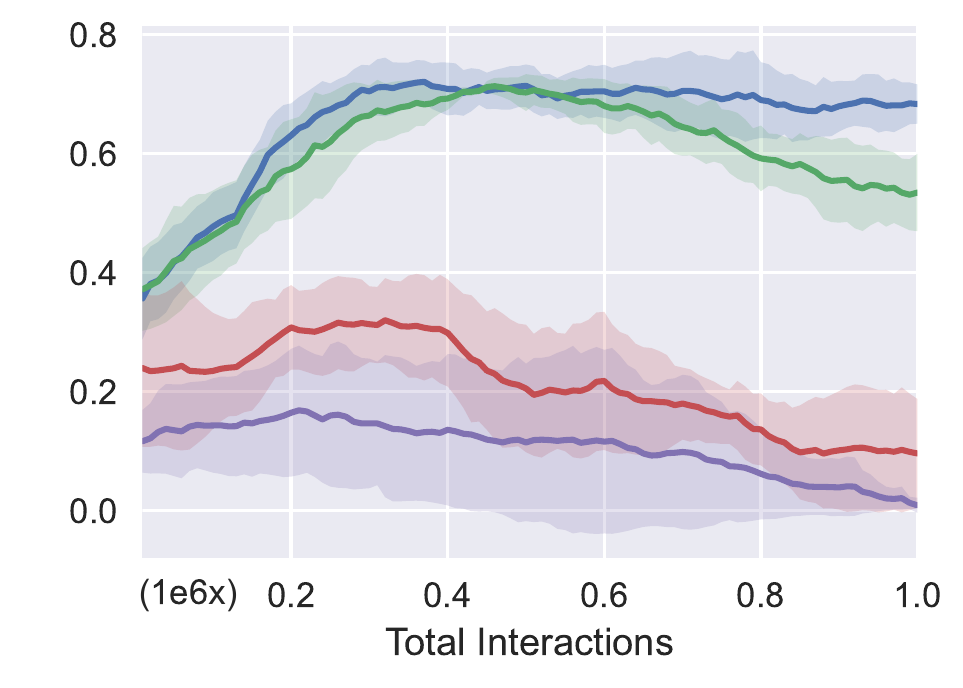}}
    \subcaptionbox{Cost-FAC-MetaDrive}
        {\includegraphics[width=0.23\linewidth,trim=20 20 0 0,clip]{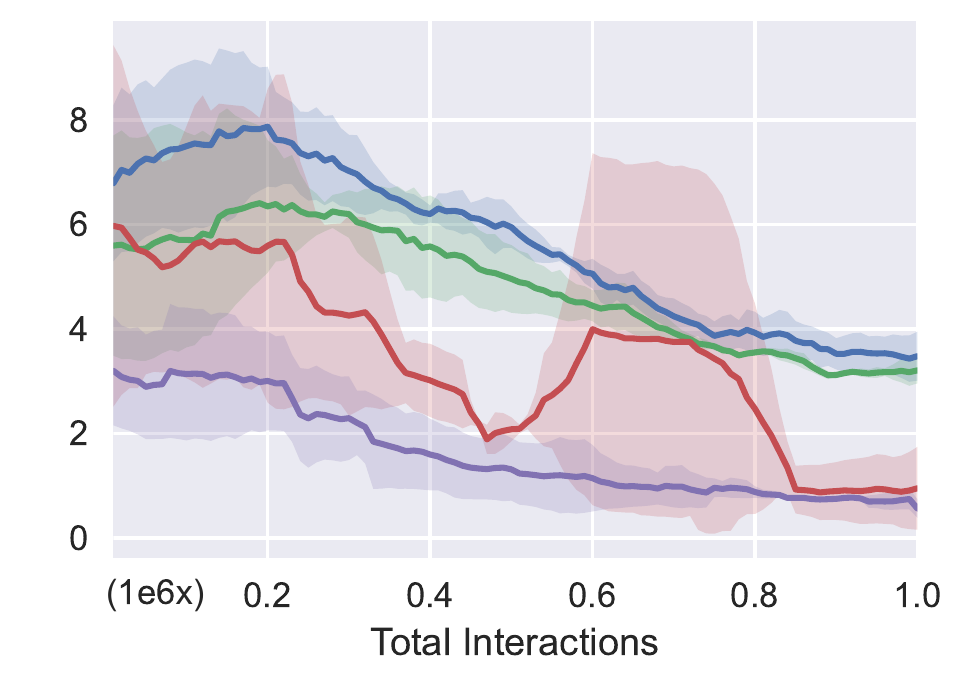}}
        \vspace{-0.15cm}
      \caption{Sensitivity studies of Lagrangian-based methods. The first two figures are reward and cost plots of Off-policy Lagrangian on SpeedLimit task with different $\lambda$ learning rates. The last two figures are success rate and cost plots of Feasible Actor-Critic on MetaDrive benchmark with different $\lambda(s;\xi)$ learning rates.}
      \label{fig:learing_curves_sa1}
      
\end{figure*}

\begin{figure*}
      \centering
    \hspace{0.5cm}\includegraphics[width=0.85\linewidth,trim=0 0 0 0,clip]{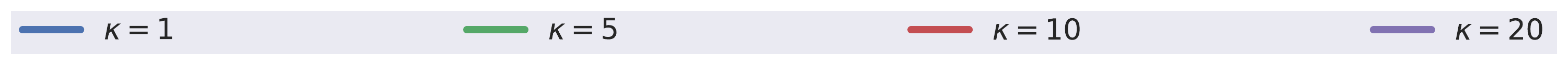}\vspace{0.1cm}\\
 \subcaptionbox{Eval Episode Reward}
        {\includegraphics[width=0.23\linewidth,trim=20 20 0 0,clip]{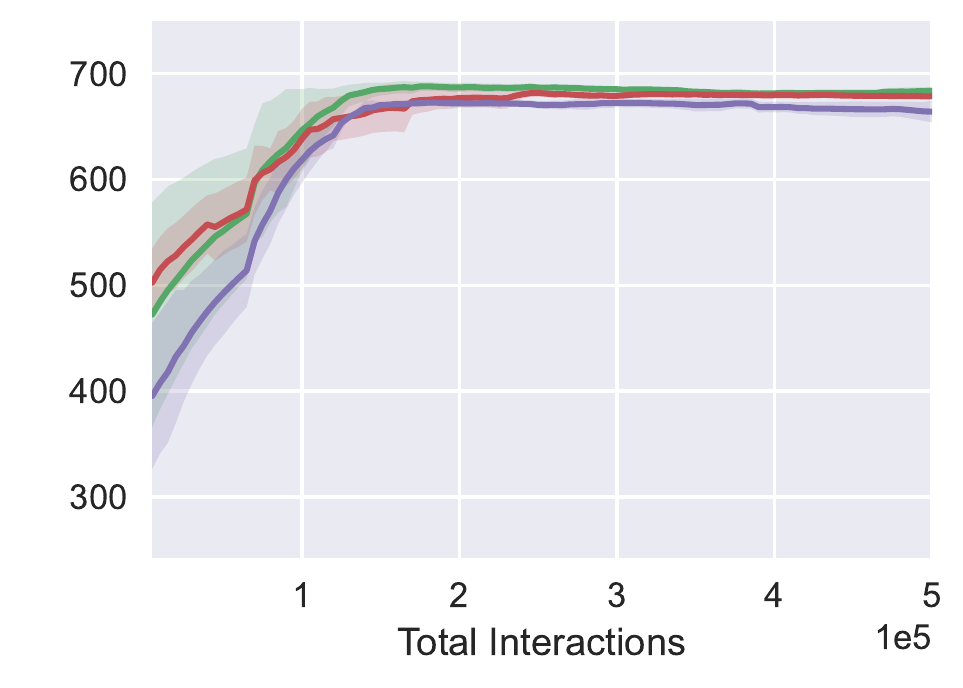}}
       \subcaptionbox{Eval Episode Cost}
        {\includegraphics[width=0.23\linewidth,trim=20 20 0 0,clip]{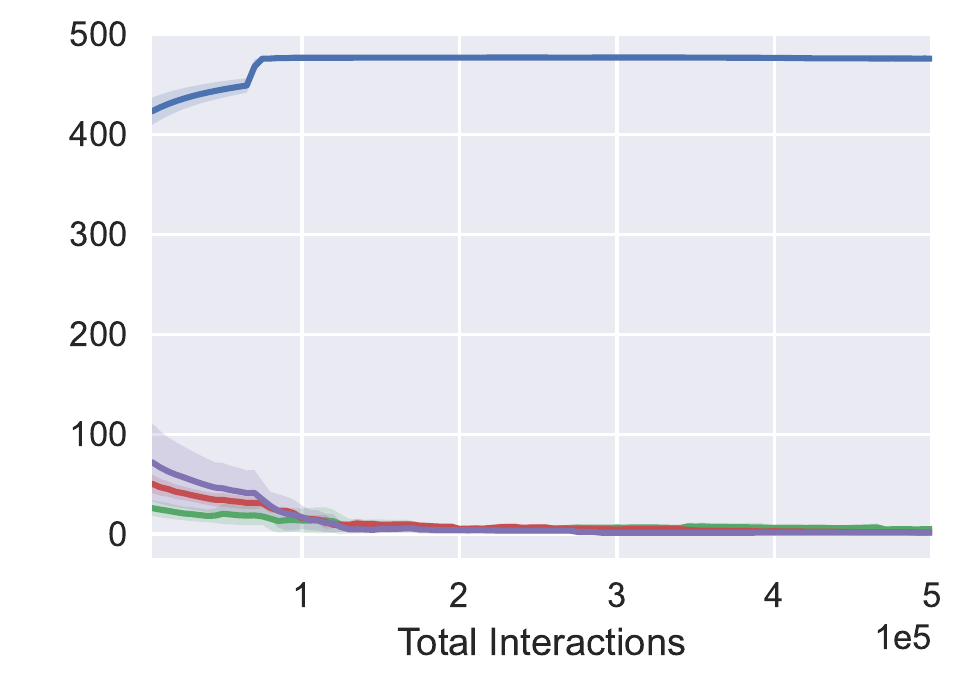}}
      \subcaptionbox{Training Cost rate}
        {\includegraphics[width=0.23\linewidth,trim=20 20 0 0,clip]{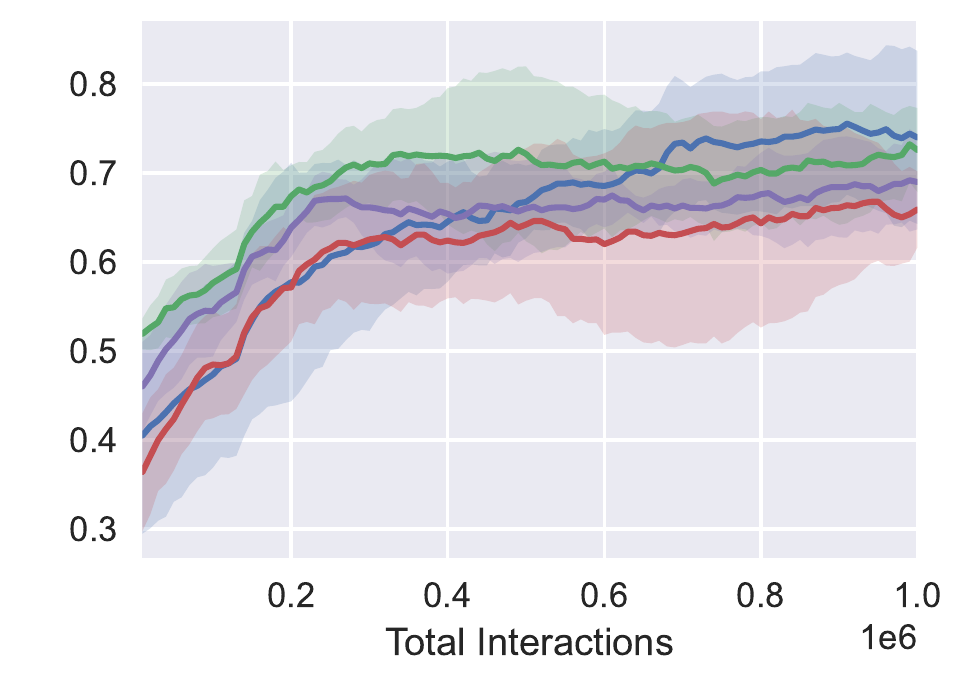}}
    \subcaptionbox{Training Cost rate}
        {\includegraphics[width=0.23\linewidth,trim=20 20 0 0,clip]{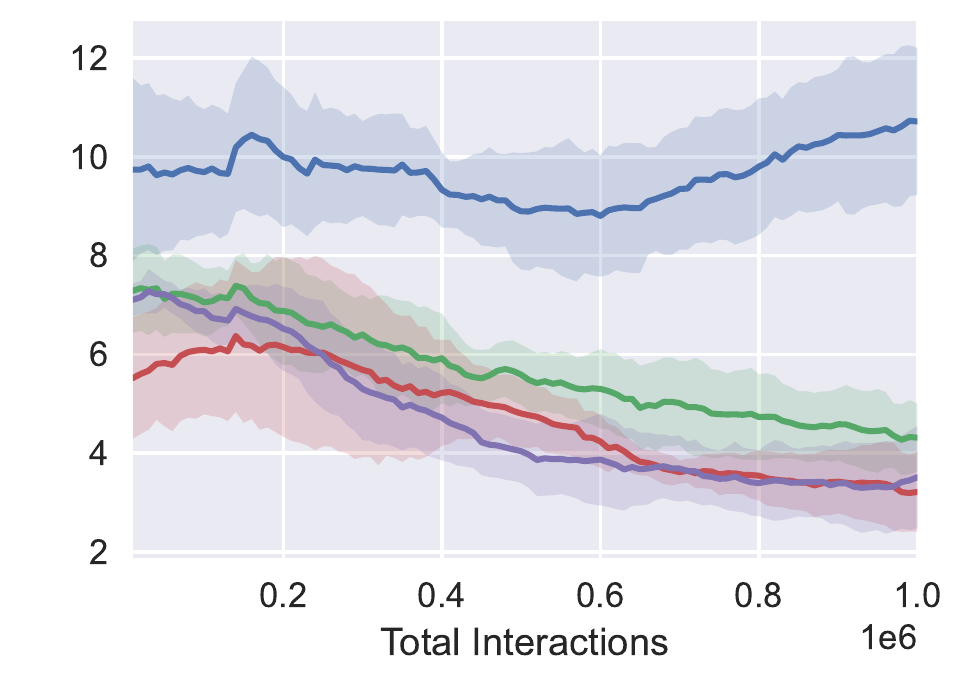}}
        \vspace{-0.15cm}
      \caption{Sensitivity studies of Exact Penalty Optimization. The first two figures are reward and cost plots of EPO on the SpeedLimit task with different penalty factors $\kappa$. The last two figures are the success rate and cost plots of EPO on the MetaDrive benchmark with different penalty factors $\kappa$.}
      \label{fig:learing_curves_sa2}
      \vspace{-0.25cm}
\end{figure*}

\clearpage

{
\small
\bibliography{example_paper}
\bibliographystyle{icml2022}
}



\end{document}